\documentclass[conference,letterpaper]{IEEEtran}
\IEEEoverridecommandlockouts
\usepackage[utf8]{inputenc} % allow utf-8 input
\usepackage[T1]{fontenc}    % use 8-bit T1 fonts
\usepackage{hyperref}       % hyperlinks
\usepackage{url}            % simple URL typesetting
\usepackage{booktabs}       % professional-quality tables
\usepackage{amsfonts}       % blackboard math symbols
\usepackage{nicefrac}       % compact symbols for 1/2, etc.
\usepackage{microtype}      % microtypography
\usepackage{xcolor}         % colors

\usepackage{url}
\usepackage{amsmath}
\usepackage{amssymb}
\usepackage{mathtools}
\usepackage{amsthm}

\theoremstyle{plain}
\newtheorem{theorem}{Theorem}[section]
\newtheorem{proposition}[theorem]{Proposition}
\newtheorem{lemma}[theorem]{Lemma}
\newtheorem{corollary}[theorem]{Corollary}
\theoremstyle{definition}
\newtheorem{definition}[theorem]{Definition}

\theoremstyle{remark}
\newtheorem{remark}[theorem]{Remark}
\usepackage[textsize=tiny]{todonotes}
\usepackage{amsfonts}
\usepackage{amssymb}
\usepackage{amsmath}
\usepackage{amsthm}
\usepackage{enumerate}
\usepackage{algorithm}
\usepackage{algorithmic}
\usepackage{tikz}
\usepackage{booktabs}
\usepackage{nicefrac}       % compact symbols for 1/2, etc.
\usepackage{makecell}
\usepackage{array, multirow, boldline}
\usepackage[subtle]{savetrees}

\newcommand{\rad}{\hat{\mathfrak{R}}}
\newcommand{\radpop}{\mathfrak{R}}

\newcommand{\radpopj}{\mathfrak{R}^{(j)}}

\newcommand{\gau}{\hat{\mathcal{G}}}
\newcommand{\gaupop}{\mathcal{G}}

\newcommand{\gaupopj}{\mathcal{G}^{(j)}}

\newcommand{\E}{\mathbb{E}}
\newcommand{\cD}{\mathcal{D}}
\newcommand{\cW}{\mathcal{W}}
\newcommand{\cX}{\mathcal{X}}
\newcommand{\cY}{\mathcal{Y}}
\newcommand{\cZ}{\mathcal{Z}}

\newcommand{\cH}{\mathcal{H}}
\newcommand{\cN}{\mathcal{N}}

\newcommand{\I}{\mathbb{I}}
\newcommand{\Dbar}{\Tilde{\mathcal{D}}}
\newcommand{\R}{\mathbb{R}}
\newcommand{\norm}[1]{\lvert\lvert #1\rvert\rvert}

\newcommand{\ip}[2]{\langle #1,#2\rangle}

\newcommand{\tD}{\Tilde{\mathcal{D}}}

\newcommand{\hf}{\widehat{f}}

\title{\textbf{Learning and Generalization with Mixture Data}}
\begin{document}
\title{Learning and Generalization with Mixture Data}

\author{
\IEEEauthorblockN{Harsh Vardhan
\IEEEauthorblockA{Computer Science \& Engineering,\\
University of California, San Diego,\\
La Jolla, CA, USA.\\
Email: hharshvardhan@ucsd.edu}}
\and
\IEEEauthorblockN{Avishek Ghosh
\IEEEauthorblockA{Computer Science \& Engineering,\\
Indian Institute of Technology, Bombay,\\
Mumbai, Maharashtra, India.\\
Email: avishek\_ghosh@iitb.ac.in}}
\and
\IEEEauthorblockN{Arya Mazumdar
\IEEEauthorblockA{Hal\i c\i o\u glu Data Science Institute,\\
University of California, San Diego,\\
La Jolla, CA, USA.\\
Email: arya@ucsd.edu}}
}

\maketitle

\begin{abstract}
In many, if not most, machine learning applications the training data is naturally heterogeneous (e.g. federated learning, adversarial attacks and domain adaptation in neural net training). Data heterogeneity is identified as one of the major challenges in modern day large-scale learning. A classical way to represent heterogeneous data is via a mixture model. In this paper, we study generalization performance and statistical rates when data is sampled from a mixture distribution. We first characterize the heterogeneity of the mixture in terms of the pairwise total variation distance of the sub-population distributions. Thereafter, as a central theme of this paper, we characterize the range where the mixture may be treated as a single (homogeneous) distribution for learning. In particular, we study the generalization performance under the classical PAC framework and the statistical error rates for parametric (linear regression, mixture of hyperplanes) as well as non-parametric (Lipschitz, convex and H\"older-smooth) regression problems. In order to do this, we obtain Rademacher complexity and (local) Gaussian complexity bounds with mixture data, and apply them to get the generalization and convergence rates respectively. We observe that as the (regression) function classes get more complex, the requirement on the pairwise total variation distance  gets stringent, which matches our intuition. We also do a finer analysis for the case of mixed linear regression and provide a tight bound on the generalization error in terms of heterogeneity.
\end{abstract}

\vspace{-1mm}
\section{Introduction}
\label{sec:intro}
\vspace{-1mm}
It is a common assumption in elementary machine learning that ``data'', which comprise of tuples of \texttt{(feature, label)}, are generated IID (identical and independently distributed) from some distribution $\cD$. Conditions of learning depend on the properties of $\cD$, and rigorous statistical performances (either via so-called uniform convergence bounds or sharp statistical error rates via M-estimation) are extremely well-studied ~\cite[Ch.~4 and 13]{wainwright2019high}, \cite{vaart_1998}.

However, often times in practice, data sources are diverse, and the IID assumption fails in many ways. In some cases, data may come from different sources, and therefore follow different distributions. Due to other benefits we may still want to learn a single model using the combined training data. 
A natural setup where \emph{non-iid} data naturally arises is in the very popular federated learning framework \cite{mcmahan2016communication,mcmahan2017federated}, where data lies in users' personal devices, and since no two users are identical, one needs to learn with heterogeneous data.
It may be the case that data is in fact coming from a mixture distribution, where components represents sub-populations (such as medical and biological data, recommendation systems etc.~\cite{covington2016deep}).

In statistics, a standard way to represent the presence of sub-populations within an overall population is the use of mixture models. Mixture of Gaussians have been extremely well-studied from a theoretical as well as an algorithmic lens over the last several decades ~\cite{dasgupta1999learning}. 
Apart from the unsupervised setup, the mixture models have been investigated for several supervised machine learning problems as well. Most prominent among these is the problem of {\em mixed linear regression}~\cite{de1989mixtures}. In this setting, the label is  generated by randomly picking a hyperplane uniformly from a set of finitely many hyperplanes. The goal here is to recover the underlying regressors from observations. This problem has received considerable attention in the recent past, both in the framework where a generative structure is assumed on the label~\cite{yi2014alternating,yi2016solving,balakrishnan2017statistical,klusowski2019estimating,krishnamurthy2019sample,pal2020recovery,gandikota2020recovery} as well in the non-generative framework, where given (covariate, label) pairs, the objective is to fit a finite number of lines~\cite{ghosh_mixed}.

However, although statistical learning with mixture data has been a topic of interest over the last few decades, most advances are made either for the parameter recovery type problems in the unsupervised setup (e.g. Gaussian mixtures) or in a relatively simple supervised setup like mixed linear regression. In this paper, we study a statistical learning perspective under the assumption that the data is sampled from a mixture distribution. Moreover, we work in the supervised learning paradigm. 

We  start by characterizing the heterogeneity in a mixture distribution in a \emph{generic} way.  Suppose the mixture distribution is $\Tilde{\mathcal{D}}$, which is defined as 
$\tilde{\cD} = \sum_{j\in [m]} a_j\cD_j$, where $\{\cD_j\}_{j=1}^m$ are the $m$ base distributions on $\cX \times \cY$, and $\{a_j\}_{j=1}^m$ are the $m$ mixture weights, with $a_j \geq 0, \, \forall j \in [m]$ and $\sum_{j=1}^m a_j = 1$. Then a natural way to characterize  the heterogeneity in $\Tilde{\mathcal{D}}$ 
is via an appropriate statistical distance between the component distributions. If this distance is large, then the statistical behavior with $\Tilde{\mathcal{D}}$ deviates a lot from any of the base distributions. On the other hand, if the distance is small, then $\Tilde{\mathcal{D}}$ approximately behaves like a homogeneous distribution altogether.

\begin{definition}($\gamma$-Heterogeneous Mixture)\label{assump:tv}

Suppose, for the $j$th component distribution, we define, $\gamma_j \equiv \norm{\cD_j -\tilde{\mathcal{D}}}_{TV}, j \in [m],$ where $\|P-Q\|_{TV}\equiv \sup_{A\subseteq \Omega}(P(A)-Q(A))$ stands for the total variation (TV) distance between the distributions $P$ and $Q$ with common sample space $\Omega$; and $\gamma\equiv \max_j \gamma_j.$
\end{definition}
Note that $\gamma$ in the above definition is the TV radius of the set of base distributions. Instead we could have defined $\gamma$ as the diameter of the space, %$\max_j \|\tilde{\mathcal{D}}- \mathcal{D}_j\|_{TV}$, 
and all our results will still  hold.
Throughout the paper, we assume that the  mixture distribution $\Tilde{\cD}$ is unknown and the learner can draw samples from $\Tilde{\cD}$. 

Our goal is to learn functions from $\cX \rightarrow \cY$ for $\cX \subseteq \mathbb{R}^{d}$, that best fits a mixture  distribution $\Tilde{\cD}$ over $(\cX, \cY)$. We assume that the learner is given access to $n$ data samples $\{x_i, y_i\}_{i=1}^{n}$ from $\Tilde{\cD}$, with maximum pairwise TV distance of $\gamma$. As usual, in the PAC setting, there exists a hypothesis  class (or mapping class) $\mathcal{F}: \cX \rightarrow \cY$, and our goal is to learn $\mathcal{F}$ through samples.

Let $\ell(f(x),y)$ be the loss function with datapoint $(x,y)$ and mapping $f \in \mathcal{F}$. Also, let $L(f) = \mathbb{E}_{(x,y)\sim \Tilde{\mathcal{D}}}\ell(f(x),y)$ be the population loss and $L_n(f) = \frac{1}{n}\sum_{i=1}^n \ell(f(x_i),y_i)$ be the empirical loss. The goal in generalization is to minimize the `excess risk' given by $L(\hat{f}) - L(f^*)$, where $\hat{f} = \mathrm{argmin}_{f\in \mathcal{F}}   L_n (f) $ and $f^* = \mathrm{argmin}_{f\in \mathcal{F}} L(f) $. It is easy to see that the excess risk is upper-bounded by 2 $ \sup_{f\in \mathcal{F}}|L_n(f) - L(f)|$, known as the generalization error.

Since we will be concerned only about generalization error, a good setting for our results is to assume, 
\begin{equation}\label{eq:same_model}
\arg\min_{f \in \mathcal{F}} \E_{\cD_j}\ell(f(x), y) = f^\star, \forall j.
\end{equation}

This ensures that we only consider the generalization error due to data heterogeneity, and not error resulting from {\em model mismatch}, though as we will see in 
Section~\ref{sec:mix_of_hyperplanes}, the latter can also contribute to total risk.
\vspace{-2mm}
\subsection{Summary of Contributions}
\vspace{-1mm}
In this paper, we aim to characterize several learning theoretic objectives when the learner has samples from the mixture distribution $\tD$. We study how learning from mixtures affects the performance in terms of (a) generalization in the PAC framework, and (b) statistical error rates in the least square regression in both parametric and non-parametric setup. 

\textbf{Generalization with mixtures:} In Section~\ref{sec:gen}, we start with the generalization performance with $\Tilde{\mathcal{D}}$ by controlling the Rademacher complexity in the standard PAC framework. In particular, we assume that the datapoints (covariate, label) $(x_i,y_i)_{i=1}^n \sim \Tilde{\mathcal{D}}$, where $\Tilde{\mathcal{D}}$ is a $\gamma$-heterogeneous mixture distribution. We first control the Rademacher complexity (corresponding to the underlying hypothesis class that maps the covariates to labels) with $\Tilde{\mathcal{D}}$. Then, by appropriately defining the generalization error and connecting it to the Rademacher complexity, we obtain generalization guarantees with mixture distribution $\Tilde{\mathcal{D}}$. In particular, we obtain the range of $\gamma$ where the generalization of $\Tilde{\mathcal{D}}$ is (order-wise) same as that of a homogeneous base distribution $\mathcal{D}_j$, for any $j \in [m]$.

\textbf{Statistical rates for mixed non-parametric least squares:} In Section~\ref{sec:non_param}, we obtain the statistical rates of several regression problems with least squares objective. We start with a simple linear regression, and we move to the more generic statistical problem of non-parametric regression. In particular we study Lipschitz ($\mathcal{F}_{Lip}$), Convex-Lipschitz ($\mathcal{F}_{conv}$) and $\alpha$-H\"older smooth ($\mathcal{S}_{\alpha}$) regression problems \cite{vaart_1998,wainwright2019high} where the covariates are drawn from a $\gamma$-heterogeneous mixture $\Tilde{\mathcal{D}}$, and  we obtain sharp statistical rates (also called prediction error) for such problems. In order to obtain this, we first characterize the local Gaussian complexity with mixed data corresponding to these non-parametric function classes; and then using the connection of Gaussian complexities with prediction error \cite[Chap. 13]{wainwright2019high}, we obtain the statistical error of these problems. Maintaining the theme of this paper, we also characterize the range of $\gamma$ where the statistical rate with $\Tilde{\mathcal{D}}$ matches to that of any (homogeneous) base distribution.

We observe that as the (regression) function classes get more complex, the requirement on $\gamma$ to maintain the same statistical rate gets stringent, which matches our intuition. Indeed, the $\gamma$ dependence strictly gets stronger, from $\gamma \sim n^{-1/3}$ for Lipschitz regression to $\gamma \sim n^{-4/5}$ for Convex-Lipschitz to $\gamma \sim n^{-\frac{2\alpha }{1+2\alpha}}$ for general $\alpha$-H\"older smooth with $\alpha >2$.

\textbf{Statistical rates for mixed linear least squares:} We now focus on a special case of heterogeneous data, where we do not assume \eqref{eq:same_model} to be true,  and do a finer analysis of the well-known mixed linear regression problem (Section~\ref{sec:mix_of_hyperplanes}). Unlike past works on mixed linear regression, which focuses on algorithmic (Alternating Minimization(AM) or Expectation Maximization (EM)) solutions, our focus here is to obtain \emph{excess risk} (generalization) result which is agnostic to any learning algorithms. We consider the setup where each base distribution $\mathcal{D}_j$ for each $j \in [m]$, corresponds to a different linear regressor, and we aim to learn a global model that minimizes a population loss (where the expectation is taken with respect to $\Tilde{\mathcal{D}}$). Notice that this framework is different that that of the special case of linear regression in Section~\ref{sec:non_param}, where the covariates are drawn from mixture distribution but we are only one regressor in the problem. Our objective is to show here that heterogeneity (which is due to model mismatch here)
affects the excess error rate. We characterize this in terms of $\gamma$, and moreover obtain ranges of $\gamma$, where this sample complexity (order-wise) matches to that of any base distribution.

\vspace{-2.5mm}
\subsection{Other related work}
\label{sec:related_work}
\vspace{-1.5mm}
In a supervised setting, heterogeneous structure like mixture of linear regression (e.g.~\cite{chaganty2013spectral,li2018learning, yi2014alternating,yi2016solving}) have received considerable interest. Several algorithms are proposed and analyzed addressing the problem of parameter recovery in mixtures. Typically two types of algorithms are used for these kind of problems : (a) Alternating Minimization (AM) (or its soft variant Expectation Maximization (EM))  \cite{balakrishnan2017statistical,klusowski2019estimating,yi2014alternating,max_affine} or (b) gradient descent \cite{li2018learning}. AM type algorithms are more attractive since they can converge in a super-linear speed \cite{ghosh2020alternating}. Very recently, in \cite{ghosh_mixed}, the authors propose a non-generative framework for mixed linear regression, and use mixed data framework for prediction in supervised learning. Using a min-loss, the paper uses a gradient-AM algorithm to find multiple linear regressors in the given data.

\noindent \emph{Notation:} We use $[r]$  to represent the set $\{1,\ldots,r\}$ and $\|.\|$ for $\ell_2$ norm for a vector, unless otherwise specified.  
 Moreover $a \lesssim b$ implies $a \leq Cb$ for some constant $C>0$, and $a \gtrsim b$ implies $a \geq Cb$. Also, we use $\Tilde{\mathcal{O}}(.)$ to ignore logarithmic factors. A random vector $x\in \R^d$  with $\E[x] = \mu$ is $\sigma^2$-sub-Gaussian if  it satisfies $\E[\exp(\ip{u}{x - \mu})] \leq \exp(-\norm{u}^2 \sigma^2/2)$. If $d=1$, then $x$ is a $\sigma^2$-sub-Gaussian random variable.

\section{Generalization with Mixture Data}
\label{sec:gen}
In this section, we provide generalization guarantees where the data is sampled from a mixture distribution. We adhere to the standard PAC setup of learning theory. 

Define the composite class $\cH = \ell \circ \mathcal{F} = \{ (x,y) \rightarrow \ell(f(x),y) : f \in \mathcal{F}\}$. It turns out that the generalization error is connected to a complexity measure, Rademacher complexity \cite[Chap. 4]{wainwright2019high}, which we now study.

\noindent\emph{Rademacher Complexity with Mixtures:}
In order to obtain generalization error guarantees on $\cH$, Rademacher complexity is a standard tool. In particular, with $n$ samples drawn from distribution $\Tilde{\cD}$, the empirical and population Rademacher complexity of $\cH$~\cite{mohri2018foundations,wainwright_2019} is defined as, 
\begin{align*} \small
    \rad_{n}(\cH) = \frac{1}{n}\E_{\mathbf{\sigma}}[\sup_{h \in \cH}  |\sum_{i=1}^{n} \sigma_i h(x_i,y_i)  | ], \radpop_{n}(\cH) = \mathbb{E}_{\Tilde{\cD}} \rad_{n}(\cH) .
\end{align*}
\normalsize where $\mathbf{\sigma}_i's$ are a set of independent Rademacher RV's. We can similarly define the Rademacher complexities for the base distributions $\{\cD_j\}$ of $\tilde{\cD}$, where for the $j$-th distribution the expectation is taken (population) or the samples are drawn (empirical) from  $\cD_j$.
\subsection{Rademacher complexity bound for mixtures}
In this section, we obtain Rademacher complexity bounds for mixture distribution $\Tilde{\cD}$. In particular we assume the mixture distribution $\Tilde{\cD}$ to be $\gamma$-heterogeneous (Definition~\ref{assump:tv}). With this, we have the following result.
\begin{proposition}
\label{lem:rad_mix}
Suppose  $\cH$ denotes the hypothesis class and let $B(n)$ denotes the upper bound on $\rad_{n}(\cH)$. The population Rademacher complexity satisfies
$
    \radpop_{n}(\cH) \leq \min_{j \in [m]}\big(\radpopj_{n}(\cH) + 2 \gamma_j B(n)\Big)
$,  where $\radpopj_{n}(\cH)$ denotes the Rademacher complexity assuming the data is drawn (in i.i.d fashion) from  $\cD_j$ for all $j \in [m]$.
\end{proposition}
So, it is clear that in order to control the rademacher complexity of the mixture, one requires to upper-bound the empirical rademacher complexity $\rad_{n}(\cH)$ (i.e., $B(n)$). We now show some special cases where this can be done.

\emph{Case I: Linear $\ell_2$-regularized Hypothesis class:} We consider the hypothesis class $\mathcal{H}=\{ h(x) = \langle w,x \rangle| \|w\| \leq W_2, \|x\| \leq R \}$. We have $\rad_{n}(\cH) \leq \frac{W_2 R}{\sqrt{n}}$ and hence $B(n) = \frac{W_2 R}{\sqrt{n}}$.

\emph{Case II: Linear $\ell_1$-regularized Hypothesis class:} We consider the hypothesis class as $\cH =\{ h(x) = \langle w, x \rangle | \|w\|_1 \leq W_1, \|x\|_\infty \leq X_\infty \}$. With this we have $\rad_n(\mathcal{H}) \leq X_\infty W_1\sqrt{\frac{2 \log \, d }{n}} = B(n)$.

\emph{Case III: Bounded Hypothesis class:} We now consider $\mathcal{H}=\{h(.)| \|h(.)| \leq b \}$. Here, we obtain $B(n) = b$.

Proof of these cases are deferred to the Appendix.
The above statement connects the Rademacher complexity with mixture samples to the same with samples drawn from a particular base distribution (here $j$-th distribution, $\cD_j$). This reduction is often useful since it is easy to work with base distributions compared to a mixture.
\begin{corollary}
Let $B(n)$ be an upper bound on empirical Rademacher complexity of $\mathcal{H}$. If for any $j \in [m]$, $\gamma_j \leq \frac{\radpopj_{n}(\cH)}{2B(n)}$, the Rachemacher complexity with mixture distribution $\tilde{\cD}$ satisfies $\radpop_{n}(\cH) \leq 2\radpopj_{n}(\cH)$. 
\end{corollary}
\textbf{Discussion:} Note that based on the decay of $\radpopj_{n}(\cH)$ with $n$, we may have different tolerable values of $\gamma$. For example, if $\radpopj_{n}(\cH) \asymp B(n)$, we may have $\gamma_j = \mathcal{O}(1)$.  

The above corollary allows us to connect the generalization (upper) bounds of a mixture learning to the same where data is drawn from a single distribution. 
\subsection{Implications for Generalization error}
We continue with the framework, where data-points $\{x_i,y_i\}_{i=1}^n$ are drawn from a $\gamma$-heterogeneous distribution $\Tilde{\mathcal{D}}$, and let $\cD_j$ denote the $j$-th base distribution. We have the following:
\begin{theorem}
\label{lem:gen_err_mix}
Suppose the composite class $\mathcal{H}$ consists of $B'$-bounded functions such that $\rad_{n}(\cH) \leq B(n)$. Then, with probability at least $1-\exp(-\frac{n\delta^2}{2(B')^2})$, we have $\sup_{f\in \mathcal{F}}|L_n(f) - L(f)| \leq 2 \radpopj_{n}(\cH) + 4B(n)\,\gamma_j + \delta$ for any $j \in [m]$. Furthermore, if  for any $j \in [m], \gamma_j \leq \frac{\radpopj_{n}(\cH)}{2B(n)}$, the generalization error is $\mathcal{O}(\radpopj_{n}(\cH))$ with high probability.
\end{theorem}
We compare this with the homogeneous setup where datapoints $\{x_i,y_i\}_{i=1}^n$ are drawn from the $j$-th base distribution $\mathcal{D}_j$. In that case, using a lower bound on generalization error \cite[Chap.4]{wainwright2019high}, we have
\begin{align*}
  \frac{1}{2} \radpopj_{n}(\cH)  - \frac{B(n)}{2\sqrt{n}} - \delta \,\,\, \leq \sup_{f\in \mathcal{F}}|L_n^{(j)}(f) - L^{(j)}(f)| \leq \,\,\, 2 \radpopj_{n}(\cH)  + \delta,
\end{align*}
with probability at least $1-2\exp(-\frac{n\delta^2}{2(B')^2})$, where $L_n$ and $L$ are empirical and population loss functions with distribution $\mathcal{D}_j$. The above implies that the generalization error is concentrated (upper and lower bounds) at $\Theta (\radpopj_{n}(\cH))$. Hence, we have following conclusion.
\begin{corollary}
If for any $j \in [m]$, $\gamma_j \leq  \frac{\radpopj_{n}(\cH)}{2B(n)}$, then, with high probability, we have 
$\sup_{f\in \mathcal{F}}|L_n(f) - L(f)| \lesssim \sup_{f\in \mathcal{F}}|L_n^{(j)}(f) - L^{(j)}(f)|$.
\end{corollary}
\begin{remark}
Provided $\gamma_j$ is small enough, the generalization error,  obtained from mixture data $\Tilde{\mathcal{D}}$ is (order-wise) no-worse than the generalization error yielded from learning via a homogeneous base distribution $\mathcal{D}_j$.The simple calculations above characterizes a level of heterogeneity below which the mixture distribution may be treated like a single distribution. With this, we now have a sufficient condition about when to treat a mixture distribution as a single base distribution.
\end{remark}

\section{Statistical Rates of Mixed Data Least Squares}
\label{sec:non_param}
\vspace{-1mm}
In this section, we characterize the performance of least squares estimation problems on a large class of functions including non-parametric functions when the data points are sampled from a mixture distribution $\tD$. Let us formalize the setup first. The learner is given access to data points $(x_i,y_i)_{i=1}^n$, where the covariates, $\{x_i\} \sim \tD$. We adhere to the standard model where the labels $y_i$ comes from the following model: 
$ y_i = f^*(x_i) + \xi_i $
for $i=1,2,\ldots,n$, where $f^*(.) \in \mathcal{F}$ is the unknown regressor function, and $\xi_i \sim \mathcal{N}(0,\zeta^2)$. The goal here is to estimate $f^*$ through $n$ samples. A natural estimator (which is also the maximum likelihood estimator in this setting) is the least squares estimator given by $\hf \in \mathrm{argmin}_{f \in \mathcal{F}} \lbrace \frac{1}{n}\sum_{i=1}^n \left(y_i - f(x_i) \right)^2 \rbrace$. A standard metric to characterize the quality of the estimator is \emph{in-sample} prediction error given by
\vspace{-1mm}
\begin{align*} \small
    \frac{1}{n} \sum_{i=1}^n \left( \hf(x_i) - f^*(x_i) \right)^2 \equiv \| \hf -f^*\|^2_n,
    \end{align*} \normalsize
where we use  $\|.\|_n$ for shorthand. A key quantity to obtain bounds on $\| \hf -f^*\|_n$ is the empirical and the population \emph{local} Gaussian complexity, given respectively by \small
\begin{align*}
    \gau_n (\delta, \mathcal{F}^*) = \E_w [ \sup_{g \in \mathcal{F}^*, \|g\|_n \leq \delta}  |\frac{1}{n} \sum_{i=1}^n u_i g(x_i) | ], 
    \gaupop_n (\delta, \mathcal{F}^*) = \E_{\tD} \gau_n (\delta, \mathcal{F}^*),
\end{align*} \normalsize
where $\mathcal{F}^*$ denotes the shifted function class given by $\mathcal{F}^* = \{ f-f^*| f\in \mathcal{F}\}$, $u_i\sim \cN(0,1)$, and $\delta$ is the local radius. 
%Moreover, $\|g\|_n = \sqrt{\frac{1}{n}\sum_{i=1}^n g(x_i)^2}$.

We first characterize the (population) local Gaussian complexity when the data points are drawn from the mixture distribution $\tD$. The tightest bounds on prediction error can be bounded using local Gaussian complexity with the local radius $\delta$ satisfying a specific critical equation~\cite[Eq (13.17)]{wainwright2019high}. Note that obtaining upper bounds directly on $\gaupop_n(\delta,\mathcal{F}^*)$ is quite non-trivial since the data is coming from a mixture distribution. So, we first connect the Gaussian complexity of a mixture to that of an individual base distribution.
\begin{proposition}
\label{lem:gauss_com_mix}
Suppose $\tD$ is $\gamma$-heterogeneous (Definition~\ref{assump:tv}) with $j$-th base distribution $\cD_j$. We have $\gaupop_n(\delta,\mathcal{F}^*) \leq \min_{j\in [m]}\big(\gaupopj_n(\delta,\mathcal{F}^*) + 2\zeta \delta \gamma_j \big),$ where $\gaupopj_n(\delta,\mathcal{F}^*)$ denotes the local Gaussian complexity with data drawn from $ \cD_j$ for any $j \in [m]$.
\end{proposition}
The above bound is particularly interesting since it upper bounds the local Gaussian complexity of a mixture with one of the base distributions $\cD_j$, and we may use standard tools from empirical process theory to upper bound the Gaussian complexity with distribution $\cD_j$.
\begin{corollary}
For a $\gamma$-heterogeneous mixture, if for any $j \in [m]$, $\gamma_j \leq \frac{\gaupopj_n(\delta,\mathcal{F}^*)}{2\zeta \delta}$, then the local Gaussian complexity satisfy
$
    \gaupop_n (\delta,\mathcal{F}^*) \leq 2 \gaupopj_n(\delta,\mathcal{F}^*).
$
\end{corollary}
\vspace{-1.5mm}
\subsection{Special Cases}
\label{sec:special_cases}

\noindent\emph{(Parametric) Linear Regression:}
We consider the function class $\mathcal{F}_{Lin} = \{ f:\R^d \rightarrow \R |  \hspace{1mm} f(x) = \langle x, \theta \rangle, \hspace{1mm} \theta \in \R^d  \}$. We compare with the setup where $\{x_i\} \sim \mathcal{D}_j$ for any $j \in [m]$. From standard results on  learning theory (see \cite{wainwright2019high} ), provided $n \geq d$, we obtain
\begin{align*}
   \frac{\zeta^2 d}{n} \lesssim  \|\hf^{(j)} - f^*\|^2_n \lesssim  \frac{\zeta^2 d}{n} \quad \text{w.p.} \hspace{2mm} 1-2\exp  (-n/\zeta^2),
\end{align*}
where $\hf^{(j)}$ is the empirical regressor when $\{x_i\}_{i=1}^n \sim \mathcal{D}_j$. Here, the lower bound is in the minimax sense. Hence, the statistical rate is tight around $\Theta(\frac{\zeta^2 d}{n} )$.
\begin{theorem}
\label{thm:lin_nonpara}
Suppose $f^* \in \mathcal{F}_{Lin}$ and data $\{x_i\}_{i=1}^n \sim \Tilde{\mathcal{D}} $. Provided, $n \geq d$, with probability at least $1-\exp  (-n/\zeta^2 )$, we have $\|\hf - f^*\|^2_n \lesssim (\delta^*)^2$, where $\delta^*$ is the solution of the critical equation $\zeta \delta \sqrt{d/n}  + 2\zeta \delta \gamma = \delta^2$. Correspondingly, provided $\gamma \leq \sqrt{d/n}$, we get $\|\hf - f^*\|^2_n \lesssim \frac{\zeta^2 d}{n}$, implying that $ \|\hf - f^*\|^2_n \lesssim \|\hf^{(j)} - f^*\|^2_n$.
\end{theorem}
\begin{remark}
When $\gamma \leq \sqrt{d/n}$, the statistical rate is (order-wise) no-worse than that of a homogeneous linear regression.
\end{remark}
\textbf{Non-parametric Classes:} We now consider several non-parametric function classes. In particular we consider Lipschitz ($\mathcal{F}_{Lip}$), Convex-Lipschitz ($\mathcal{F}_{conv}$) and $\alpha$-H\"older smooth ($\mathcal{S}_{\alpha}$) which are increasingly complex.
%such that $\mathcal{F}_{Lip} \supseteq \mathcal{F}_{conv} \supseteq \mathcal{S}_{\alpha}$. 
We show that as this complexity increases, the requirement on $\gamma$ gets stringent. 

\begin{remark}
As parametric linear regression is not Lipschitz, so the results are not directly comparable with the non-parametric classes.
\end{remark}
\noindent\emph{Lipschitz Regression:}
We consider $\mathcal{F}_{Lip} = \{ f:[0,1] \rightarrow \R | f(0) = 0, \hspace{1mm} f \hspace{2mm} \text{is} \hspace{2mm} L\text{-Lipschitz} \}$. If $\{x_i\}_{i=1}^n \sim \mathcal{D}_j$ for any $j \in [m]$, from standard results on (homogeneous) non-parametric inference (see \cite{aditya_stat},\cite[Chap. 13]{wainwright2019high} ), we obtain
\begin{align*}
   ( \frac{L\zeta^2}{n} )^{2/3} \lesssim \|\hf^{(j)} - f^*\|^2_n \lesssim ( \frac{L\zeta^2}{n} )^{2/3} \end{align*}with probability $1-2\exp  (-n^{1/3}/2\zeta^{2/3} )$,
where the lower bound is in the minimax sense. Hence, the statistical rate is tightly bounded around $\Theta((\frac{L\zeta^2}{n} )^{2/3})$ for homogeneous Lipschitz regression.
\begin{theorem}
\label{thm:lip_nonpara}
Suppose $f^* \in \mathcal{F}_{Lip}$ and data $\{x_i\}_{i=1}^n \sim \Tilde{\mathcal{D}} $. With probability at least $1-\exp  (-n^{1/3}/2\zeta^{2/3} )$, we have $\|\hf - f^*\|^2_n \lesssim (\delta^*)^2$, where $\delta^*$ is the solution of the critical equation $\zeta\sqrt{L\delta/n} + 2\zeta \delta \gamma_j = \delta^2$. Correspondingly, provided $\gamma_j \leq \left( \frac{L}{\zeta n} \right)^{1/3}$, we get $\|\hf - f^*\|^2_n \lesssim ( \frac{L\zeta^2}{n} )^{2/3}$, implying $ \|\hf - f^*\|^2_n \lesssim \|\hf^{(j)} - f^*\|^2_n$.
\end{theorem}
\noindent\emph{Convex-Lipschitz Regression:}
We characterize a stricter class of functions, namely the set of all convex Lipschitz functions $\mathcal{F}_{conv} = \{ f:[0,1] \rightarrow \R | f(0) = 0, \hspace{1mm} \, f \hspace{2mm} \text{is} \hspace{2mm} \text{convex and } 1-\text{Lipschitz} \}$. If $\{x_i\}_{i=1}^n \sim \mathcal{D}_j$ for any $j \in [m]$, and again using some standard results, \cite{aditya_stat},\cite[Chap. 13]{wainwright2019high}, we obtain
\begin{align*}
   ( \frac{\zeta^2}{n} )^{4/5} \lesssim \|\hf^{(j)} - f^*\|^2_n \lesssim ( \frac{\zeta^2}{n} )^{4/5} \end{align*}with probability  $1-2\exp (-\frac{n}{\zeta^2} (\frac{\zeta^2}{n})^{\frac{4}{5}} )$,
where the lower bound is in the minimax sense. 
\begin{theorem}
\label{thm:cvx_nonpara}
Suppose $f^* \in \mathcal{F}_{conv}$ and data $\{x_i\}_{i=1}^n \sim \Tilde{\mathcal{D}} $. With probability at least $1-\exp (-(\frac{n}{\zeta^2})^{\frac{1}{5}})$, we have $\|\hf - f^*\|^2_n \lesssim (\delta^*)^2$, where $\delta^*$ is the solution of the critical equation $(\zeta/\sqrt{n}) \delta^{3/4} + 2\zeta \delta \gamma_j = \delta^2$. Hence, provided $\gamma_j \leq \left( \frac{1}{\zeta}\right)^{\frac{1}{5}}  \left( \frac{1}{ n} \right)^{\frac{2}{5}}$, we have $\|\hf - f^*\|^2_n \lesssim ( \frac{\zeta^2}{n} )^{4/5}$, implying $ \|\hf - f^*\|^2_n \lesssim \|\hf^{(j)} - f^*\|^2_n$.
\end{theorem}
\noindent\emph{H\"older Smooth Function Classes:}
To generalize the above mentioned function classes, we consider the H\"older smooth classes. Fix $\alpha >0$, and $\beta$ be the largest integer strictly smaller than $\alpha$. The function class $\mathcal{S}_\alpha$ consists of functions $f$ on $[0,1]$ satisfying the following:
(a) $f$ is continuous on $[0,1]$ and $\beta$ times differentiable on $(0,1)$; (b) $|f^{(k)}(.)| \leq 1$ for all $k=1,..,\beta$; and (c) $|f^{(\beta)}(x) - f^{(\beta)} (y)| \leq |x-y|^{\alpha-\beta}$ for all $x,y \in (0,1)$.

Note that if $\alpha =1$, $\mathcal{S}_1$ is just class of all bounded Lipschitz functions. For $\alpha =2$, we get the convex Lipschitz class. In general, $\alpha$ denotes the smoothness of the class.
Using bounds for statistical error (see,~\cite{aditya_stat}), if $\{x_i\}_{i=1}^n \sim \mathcal{D}_j$ for any $j \in [m]$,
\begin{align*}
   ( \frac{\zeta^2}{n} )^{\frac{2\alpha}{1+2\alpha}} \lesssim  \|\hf^{(j)} - f^*\|^2_n \lesssim ( \frac{\zeta^2}{n} )^{\frac{2\alpha}{1+2\alpha}}\end{align*}with probability $1-2\exp (-\frac{n}{\zeta^2} (\frac{\zeta^2}{n})^{\frac{2\alpha}{1+2\alpha}} )$.

So, the statistical rate is tightly bounded around $\Theta((\frac{\zeta^2}{n} )^{\frac{2\alpha}{1+2\alpha}})$.
\begin{theorem}
\label{thm:holder_nonpara}
Suppose $f^* \in \mathcal{S}_\alpha$ and data $\{x_i\}_{i=1}^n \sim \Tilde{\mathcal{D}} $. With probability at least $1-\exp (-\frac{n}{\zeta^2} (\frac{\zeta^2}{n})^{\frac{2\alpha}{1+2\alpha}}$, we have $\|\hf - f^*\|^2_n \lesssim (\delta^*)^2$, where $\delta^*$ is the solution of the critical equation $(\zeta/\sqrt{n}) \delta^{1-(1/2\alpha)} + 2\zeta \delta \gamma_j = \delta^2$. Accordingly, provided $\gamma_j \leq \left( \frac{1}{\zeta}\right)^{\frac{1}{1+2\alpha}}  \left( \frac{1}{n} \right)^{\frac{\alpha}{1+2\alpha}}$, we obtain $\|\hf - f^*\|^2_n \lesssim ( \frac{\zeta^2}{n} )^{\frac{2\alpha}{1+2\alpha}}$, implying $ \|\hf - f^*\|^2_n \lesssim \|\hf^{(j)} - f^*\|^2_n$.
\end{theorem}
\begin{remark}
The conditions on $\gamma_j$ gets more stringent when we consider more complex regression classes. For Lipschitz to Convex to H\"older smooth (which are  increasingly complex for $\alpha > 2 $), the $\gamma_j$ dependence gets strictly  stronger (from $\gamma_j \sim n^{-1/3}$ for Lipschitz to $\gamma_j \sim n^{-4/5}$ for convex to $\gamma_j \sim n^{-\frac{2\alpha }{1+2\alpha}}$ for $\alpha$-H\"older smooth).
\end{remark}

\section{Mixture of Hyperplanes}
\label{sec:mix_of_hyperplanes}
\vspace{-1.5mm}
In this section, we focus on out-of-sample prediction error of linear regression where the covariate distribution is homogeneous, but the labels are generated from a mixture of hyperplanes. The learner has access to $n$ data points $\{(x_i,y_i)\}_{i=1}^n$. 
Each base distribution $\cD_j$ corresponds to a unique hyperplane $w_j^\star$. To sample the $i^{th}$ datapoint from the mixture $\Tilde{\mathcal{D}} = \sum_{j=1}^m a_j \cD_j$, we first sample  $x_i \overset{iid}{\sim} \cN(0,\nu^2 \I_d)$. Then we sample $w$ from a set $\{w^\star_1,w^\star_2, \dots, w^\star_m\}$, such that $\Pr(w^\star_j)=a_j, j =1, \dots ,m$, the mixing weights.  Then, we generate the label $y_i$ according to $y_i = \ip{w}{x_i} + \xi_i$, where $\xi_i$ is a $0$-mean $\zeta^2$-subGaussian random variable independent of $w$ and $x_i$. 
$\Tilde{\mathcal{D}}$ now denotes a mixture of $m$ different hyperplanes. 
\begin{definition}[Heterogeneity]\label{def:het_mix_planes}
 The heterogeneity in the mixture of hyperplanes is characterized by $\Delta_w \triangleq \max_{j,j'\in [m]}\norm{w_{j'}^\star - w_j^\star}$ (this is exclusively due to model mismatch).
 \end{definition}
One possible method to obtain statistical rates for this mixture is to characterize the subGaussian parameter of this mixture in terms of $\Delta_w$ and plug it in the results for homogeneous distribution~(e.g., from~\cite{pmlr-v23-hsu12}). To do so, if a datapoint $y_i$ is generated as $y_i = \ip{w}{x_i} + \xi_i$ we can express it as $y_i = \ip{w_j^\star}{x_i} + (\ip{w_j^\star - w}{x_i} + \xi_i)$ where $(\ip{w_j^\star - w}{x_i} + \xi_i)$ is the noise. The noise has subGaussian parameter $\zeta^2 + \sigma^2 \Delta_w^2$ but, we cannot use this as the results for the  homogeneous case  crucially require the noise to be either independent of $x_i$ or bounded almost surely both of which are unsatisfied here~\cite{pmlr-v23-hsu12}. In the following, we  resort to an analysis that uses the special properties of this mixture to obtain the statistical error guarantees.

\begin{remark}\label{rem:kl} 
Note that, $\Delta_w$ denotes the maximum distance between the hyperplanes of two base distributions. This also turns out to be a bound on several natural measures of distances between base distributions, such as the KL divergence. In particular, we can connect the  heterogeneity measure in Definition~\ref{assump:tv} to pairwise TV distance, if $\xi$ were Gaussian instead of subGaussian. Indeed, using chain rule and Pinsker's inequality, 
\begin{align*} \small
\gamma &\le \max_{j,j'\in [m]}\sqrt{KL(\cD_j || \cD_{j'})/2}\\&\le \small\max_{j,j'\in [m]}\sqrt{\E_{x}[KL(\cD_j || \cD_{j'})|x]/2} \leq \frac{\nu \Delta_w}{\sqrt{2}\zeta}.
 \end{align*}
\end{remark}\normalsize
Using the square loss, $\ell(w,(x,y)) = \frac{1}{2}(\ip{w}{x} - y)^2$, we learn an estimator which minimizes the empirical risk given by $\hat{w} = \arg\min_{w\in \R^d} \frac{1}{n}\ell(w,(x_i,y_i))$. To compute the out-of-sample prediction error, we define the  population risk for mixture $\Dbar$.
\begin{equation}\label{eq:pop_risk_mix}
\begin{aligned}
    &F(w) = \E_{(x,y)\sim \tilde{\cD}}[\ell(w,(x,y))],\\
    &w^\star  \equiv \arg\min_{w\in \R^d} F(w) = \sum_{j=1}^m a_jw_j^\star
\end{aligned}
\end{equation}
The out-of-sample prediction error of $\hat{w}$ is    $F(\hat{w}) - F(w^\star)$. 
\begin{theorem}\label{thm:mix_planes_linreg} Let $\delta >0$.
For linear regression on mixture of hyperplanes $\tilde{\cD}$, where heterogeneity is defined by Definition~\ref{def:het_mix_planes}, and sample complexity $n > n_{\delta} = \Theta(d(\log d + \log(1/\delta))$, with probability atleast $1-\delta$, the out-of-sample prediction error is
 \small 
 \begin{align*}
F(\hat{w}) - F(w^\star) 
& = \mathcal{O}(\frac{\zeta^2(d + \sqrt{d \log(1/\delta)} +  \log(1/\delta))}{n} +\frac{d\nu^2\Delta_w^2}{n}\log\frac{1}{\delta}) 
\end{align*}
\normalsize
\end{theorem}
The proof is in the full paper. Note that, the model mismatch led to a second term in generalization error.
We now consider the homogeneous setting. For base distribution $\cD_j$, we sample $n$ datapoints $(x_i, y_i)$,  where $x_i$ is sampled from the same distribution, $y_i$ is generated using a fixed $w_j^\star$. The distribution of noise is still the same, $0$-mean $\zeta^2$-subGaussian. Using the same loss function $\ell$, the empirical risk minimizer in this case is $\hat{w}_j$, the population risk is $F_j(w) = \E_{(x,y) \sim \cD_j}[\ell(w,(x,y))]$, its minimizer is $w_j^\star$ and the out-of-sample prediction error is $F_j(\hat{w}_j) - F_j(w_j^\star)$. Then, with probability $1-\delta$, for $n \geq n_\delta$, we obtain,
\begin{align*}
    \frac{d\zeta^2}{n} \lesssim F_j(\hat{w}_j) - F_j(w_j^\star) \lesssim \frac{d\zeta^2}{n}.
\end{align*}
The upper bound is provided in ~\cite{pmlr-v23-hsu12} and the (standard) lower bound can be found in ~\cite{vaart_1998,LehmCase98}. 
Comparing this with Theorem~\ref{thm:mix_planes_linreg}, the mixture of hyperplanes obtains strictly larger out-of-sample error than base distribution, with the difference depending on  heterogeneity $\Delta_w$. However, both the mixture and base distribution require the same number of samples $n_\delta$.
\begin{corollary}
If $\Delta_w \leq \frac{\zeta}{\nu}$, then the out-of-sample prediction error of the mixture $\tilde{\cD}$ is no worse (order-wise) than a component  $\cD_j$ , $F(\hat{w}) - F(w^\star) \lesssim F_j(\hat{w}_j) - F_j(w_j^\star)$, for any $j \in [m]$.
\end{corollary}

{\em Acknowledgment:} This research is supported in part by NSF awards 2217058 and 2112665.

\vspace{-2mm}

\bibliographystyle{IEEEtran}
\bibliography{references}

%%% Uncomment this for arxiv submission.
\newpage

\onecolumn
\appendices
\begin{center}
    \Large \textbf{Appendix for "Learning and Generalization with Mixture Data"}
\end{center}
\hrulefill
\section{Proofs for Generalization}

\subsection{Proof of Proposition \ref{lem:rad_mix} }
    We have
\begin{align*}
    \radpop_{n}(\cH) &= \mathbb{E}_{\cD_j} \rad_{n}(\cH) + (\mathbb{E}_{\Tilde{\cD}} \rad_{n}(\cH) - \mathbb{E}_{\cD_j} \rad_{n}(\cH)) \\
    & \leq \radpopj_{n}(\cH) + | \mathbb{E}_{\Tilde{\cD}} \rad_{n}(\cH) - \mathbb{E}_{\cD_j} \rad_{n}(\cH) | \\
    &  \leq \radpopj_{n}(\cH) + \int |\rad_{n}(\cH)| \,  |\mathsf{d}P_{\Tilde{\cD}} - \mathsf{d}P_{\cD_j} | \\
    & \leq \radpopj_{n}(\cH) + 2B(n) \, \gamma_j,
\end{align*}
where the last inequality comes from boundedness in $\rad_{n}(\cH)$ and Definition~\ref{assump:tv}.

\subsection{Proof of Case I: Linear $\ell_2$-regularized Hypothesis class}
\begin{proof}
We continue
\begin{align*}
   & \rad_n (\mathcal{H}) = \frac{1}{n} \E_{\sigma} \bigg [ \sup_{w: \|w\| \leq W_2}  \bigg| \sum_{i=1}^n \sigma_i \langle w,x_i \rangle \bigg | \bigg] \\
    & = \frac{1}{n} \E_{\sigma} \bigg [ \sup_{w: \|w\| \leq W_2}  \bigg| \langle w, \sum_{i=1}^n \sigma_i x_i \rangle \bigg | \bigg] 
     = \frac{W_2}{n} \E_{\sigma} \bigg [   \|  \sum_{i=1}^n \sigma_i x_i \| \bigg] \\
    & \leq \frac{W_2}{n} \sqrt{\E_{\sigma} \bigg [     \sum_{i=1}^n \|\sigma_i x_i \|^2 \bigg]} 
     = \frac{W_2}{n} \sqrt{\E_{\sigma} \bigg [     \sum_{i=1}^n \| x_i \|^2 \bigg]} 
     \leq \frac{W_2 R}{\sqrt{n}}.
\end{align*}
\end{proof}

\subsection{Proof of Case II: Linear $\ell_1$-regularized Hypothesis class}
\begin{proof}
We continue
\begin{align*}
   & \rad_n (\mathcal{H}) = \frac{1}{n} \E_{\sigma} \bigg [ \sup_{w: \|w\|_1 \leq W_1}  \bigg| \sum_{i=1}^n \sigma_i \langle w,x_i \rangle \bigg | \bigg] \\
    & = \frac{1}{n} \E_{\sigma} \bigg [ \sup_{w: \|w\|_1 \leq W_1}  \bigg| \langle w, \sum_{i=1}^n \sigma_i x_i \rangle \bigg | \bigg] 
     = \frac{W_1}{n} \E_{\sigma} \bigg [   \|  \sum_{i=1}^n \sigma_i x_i \|_\infty \bigg] \\
    & \leq \frac{W_1}{n} \sqrt{\E_{\sigma} \bigg [      \sup_j \sum_{i=1}^n \sigma_i x_i[j] \bigg]} 
     \leq \frac{W_1 \sqrt{2 \log d}}{n} \sqrt{\bigg [     \sum_{i=1}^n (x_i[j])^2 \bigg]} \\
    & \leq X_\infty W_1\sqrt{\frac{2 \log \, d }{n}},
\end{align*}
where the third line uses Cauchy Schwartz inequality, and in the fourth line, $x_i[j]$ denotes the $j$-th coordinate of $x_i$. Moreover, in the fifth line we use Massart's finite Lemma \cite{massart_finite}.
\end{proof}
\subsection{Proof of Case III: Bounded Hypothesis class}
\begin{proof}
We show that, for bounded hypothesis class $\cH$, the empirical Rademacher complexity is also bounded. We have $ \rad_{n}(\cH) = \frac{1}{n}\E_{\mathbf{\sigma}}[\sup_{h \in \cH}  |\sum_{i=1}^{n} \sigma_i h(x_i,y_i)  | ] = \frac{1}{n}\E_{\mathbf{\sigma}}[\sup_{h \in \cH}  | \langle \mathbf{\sigma}, \mathbf{h} \rangle  |]$. where $\mathbf{\sigma} = [\sigma_1 \ldots \sigma_{n}]^\top$ and $\mathbf{h} = [h(x_1,y_1) \ldots h(x_{n},y_n)]^\top$. Note that supremum occurs when $\mathbf{h}$ and $\mathbf{\sigma}$ are aligned, which implies $\rad_{n}(\cH) = \frac{1}{n}\E_{\mathbf{\sigma}} [  \| \mathbf{\sigma}\|\, \| \mathbf{h} \|   ]  \leq \frac{1}{n} \E_{\mathbf{\sigma}}  [ (\sqrt{n}) \, (b\sqrt{n}) ] = b$.
\end{proof}

\subsection{Proof of Lemma~\ref{lem:gen_err_mix}}
It is well-known that the generalization error can be upper-bounded by symmetrized argument (see \cite[Theorem 4.10]{wainwright2019high} that
\begin{align*}
    \sup_{f\in \mathcal{F}}|L_n(f) - L(f)| \leq 2 \radpop_{n}(\cH) + \delta,
\end{align*}
where $\mathcal{H}$ is the composite class. In particular we view the loss  $\ell(f(X),Y)$ as  $h(X,Y),$ where $h \in \mathcal{H}$. Invoking Lemma~\ref{lem:rad_mix} yields the result. 

\section{Proofs for Non Parametric Regression with Mixture Data}
\label{sec:non_para_proof}
We start this section by writing the following classical result about non-parameteric regression.

\begin{theorem}[Wainwright, Chap 13]
\label{thm:wain_nonpara}
Suppose $\mathcal{F}^*$ is star-shaped. Let $\delta^* $ be the smallest value satisfying $\gaupop_n(\delta,\mathcal{F}^*) \lesssim \delta^2$. Then, for any $t \geq \delta^*$, we obtain
\begin{align*}
    \|\hf - f^*\|^2_n \leq 16 t \delta^*, \hspace{2mm} \text{w.p.} \hspace{2mm} 1-\exp(-n t \delta^*/2\zeta^2)
\end{align*}
\end{theorem}

\subsection{Proof of Proposition~\ref{lem:gauss_com_mix}}
The proof goes along the same
line as the proof of Proposition~\ref{lem:rad_mix}, uses standard Gaussian norm bound and hence omitted.

\subsection{Proof of Theorem~\ref{thm:lin_nonpara}}
we look at centered one dimensional Lipschitz functions. Similarly, we can define $\mathcal{F}^*_{Lin}$ as a $f^*$ shifted version of the same. From \cite[Ch. 13]{wainwright2019high}, we obtain, $\gaupopj_n(\delta,\mathcal{F}^*_{Lin}) \lesssim \zeta \delta \sqrt{\frac{d}{n}} $. Using Proposition~\ref{lem:gauss_com_mix}, the critical equation now becomes $\zeta \delta \sqrt{d/n} + 2\zeta \delta \gamma_j = \delta^2$ (see \cite[Chap. 13]{wainwright2019high}), with the solution $\delta^*$ as the statistical rate.

Furthermore, from the above calculation and Theorem~\ref{thm:wain_nonpara}, provided, $\gamma_j \leq \sqrt{d/n}$, we write the critical equation as, $\zeta \delta \sqrt{\frac{d}{n}} \lesssim \delta^2$, or equivalently, the critical value of $\delta$, given by $\delta^* = (\zeta \frac{d}{n})$. Consequently, the condition on $\gamma_j$ turns out to be
\begin{align*}
    \gamma_j \leq \sqrt{\frac{d}{n}},
\end{align*}
and equivalently, the statistical prediction rate (from Theorem~\ref{thm:wain_nonpara}) is
\begin{align*}
     \|\hf - f^*\|^2_n \lesssim \zeta^2 \frac{d}{n}, \hspace{2mm} \text{w.p.} \hspace{2mm} 1-\exp \bigg (-n/\zeta^2 \bigg)
\end{align*}

\subsection{Proof of Theorem~\ref{thm:lip_nonpara}}
In particular, we look at centered one dimensional Lipschitz functions. Similarly, we can define $\mathcal{F}^*_{Lip}$ as a $f^*$ shifted version of the same. From \cite[Ch. 13]{wainwright2019high}, we obtain, $\gaupopj_n(\delta,\mathcal{F}^*_{Lip}) \lesssim \zeta \sqrt{\frac{L \delta}{n}} $. Using Proposition~\ref{lem:gauss_com_mix}, the critical equation now becomes $\zeta\sqrt{L\delta/n} + 2\zeta \delta \gamma_j = \delta^2$ (see \cite[Chap. 13]{wainwright2019high}), with the solution $\delta^*$ as the statistical rate.

Furthermore, from the above calculation and Theorem~\ref{thm:wain_nonpara}, provided, $\gamma_j \leq \sqrt{\frac{L}{n\delta}}$, we write the critical equation as, $\zeta \sqrt{\frac{L \delta}{n}} \lesssim \delta^2$, or equivalently, the critical value of $\delta$, given by $\delta^* = (\frac{L\zeta^2}{n})^{1/3}$. Consequently, the condition on $\gamma_j$ turns out to be
\begin{align*}
    \gamma_j \leq \left( \frac{L}{\zeta n} \right)^{1/3},
\end{align*}
and equivalently, the statistical prediction rate (from Theorem~\ref{thm:wain_nonpara}) is
\begin{align*}
     \|\hf - f^*\|^2_n \lesssim \left( \frac{L\zeta^2}{n} \right)^{2/3}, \hspace{2mm} \text{w.p.} \hspace{2mm} 1-\exp \bigg (-n^{1/3}/2\zeta^{2/3} \bigg)
\end{align*}

\subsection{Proof of Theorem~\ref{thm:cvx_nonpara}}
From \cite[Ch. 13]{wainwright2019high}, we obtain, $\gaupopj_n(\delta,\mathcal{F}^*_{conv}) \lesssim \zeta \frac{1}{\sqrt{n}} \delta^{3/4} $. 

From the above-mentioned theorem, provided, $\gamma_j \leq \sqrt{\frac{1}{\sqrt{n}}} \delta^{-1/4}$, we write the critical equation as, $\zeta \frac{1}{\sqrt{n}} \delta^{1-(1/4)} \lesssim \delta^2$, or equivalently, the critical value of $\delta$, given by $\delta^* = (\frac{\zeta^2}{n})^{\frac{2}{5}}$. Consequently, the condition on $\gamma_j$ turns out to be
\begin{align*}
    \gamma_j \leq \left( \frac{1}{\zeta}\right)^{\frac{1}{5}}  \left( \frac{1}{ n} \right)^{\frac{2}{5}},
\end{align*}
and equivalently, the statistical prediction rate (from Theorem~\ref{thm:wain_nonpara}) is
\begin{align*}
     \|\hf - f^*\|^2_n \lesssim \left( \frac{\zeta^2}{n} \right)^{\frac{4}{5}}, \hspace{2mm} \text{w.p.} \hspace{2mm} 1-\exp \bigg (-\frac{n}{\zeta^2} (\frac{\zeta^2}{n})^{\frac{4}{5}} \bigg )
\end{align*}

\subsection{Proof of Theorem~\ref{thm:holder_nonpara}}
For $\alpha$ smooth function classes, \cite{aditya_stat} obtains the following: $\gaupopj_n(\delta,\mathcal{S}^*_{\alpha}) \lesssim \zeta \frac{1}{\sqrt{n}} \delta^{1-(1/2\alpha)}$

From the above mentioned theorem, provided, $\gamma_j \leq \frac{1}{\sqrt{n}} \delta^{(-1/2\alpha)}$, we write the critical equation as, $\zeta \frac{1}{\sqrt{n}} \delta^{1-(1/2\alpha)} \lesssim \delta^2$, or equivalently, the critical value of $\delta$, given by $\delta^* = (\frac{\zeta^2}{n})^{\frac{\alpha}{1+2\alpha}}$. Consequently, the condition on $\gamma_j$ turns out to be
\begin{align*}
    \gamma_j \leq \left( \frac{1}{\zeta}\right)^{\frac{1}{1+2\alpha}}  \left( \frac{1}{ n} \right)^{\frac{\alpha}{1+2\alpha}},
\end{align*}
and equivalently, the statistical prediction rate (from Theorem~\ref{thm:wain_nonpara}) is
\begin{align*}
     \|\hf - f^*\|^2_n \lesssim \left( \frac{\zeta^2}{n} \right)^{\frac{2\alpha}{1+2\alpha}}, \hspace{0.75mm} \text{w.p.} \hspace{0.75mm} 1-\exp \bigg (-\frac{n}{\zeta^2} (\frac{\zeta^2}{n})^{\frac{2\alpha}{1+2\alpha}} \bigg )
\end{align*}
\begin{remark}
If we substitute $\alpha=1$, we get the rates of the Lipschitz class, and if $\alpha=2$, we get the same for the convex class.
\end{remark}

\section{Proofs for Mixture of Hyperplanes}
\label{sec:mix_planes_proof}

\subsection{Proof of Theorem~\ref{thm:mix_planes_linreg}}
We first provide the exact formulations for population risk of mixture and base distributions.

For each base distribution, the population risk $F_j(w)$ is given by
\begin{align*}
    F_j(w) &= \frac{1}{2}\E[(\ip{w - w_j^\star}{x} + \xi)^2] = \frac{1}{2}\E[\ip{w - w_j^\star}{x}^2] + \frac{1}{2}\E[\xi^2]\\
    &=  \frac{1}{2} (w - w_j^\star)^\intercal \E[x x^\intercal] (w - w_j^\star) + \frac{1}{2}\E[\xi^2]
    = \frac{\nu^2}{2}\norm{w - w_j^\star}^2 + \E[\xi^2]\\
    & = \frac{\nu^2}{2}\norm{w - w_j^\star}^2
\end{align*}
We use $\E[\xi] = 0$ and independence of $\xi$ and $x$ in the first equation. As $x\sim \cN(0,\nu^2 \I_d)$, $\E[x x^\intercal] = \nu^2 \I_d$. From the above expression, we can see that the minimizer of $F_j(w)$ is $w_j^\star$ and $F_j(w_j^\star) = \frac{1}{2}\E[\xi^2]$

The population risk for the mixture, $F(w)$, is given by
\begin{align*}
    F(w)  &= \E_{\Dbar}[\ell(w,(x,y)] = \sum_{j=1}^m a_j F_j(w)\\
    & = \frac{\nu^2}{2}\sum_{j=1}^m a_j \norm{w - w_j^\star}^2  + \sum_{j=1}^m a_j F_j(w_j^\star)\\
    &= \frac{\nu^2}{2}\sum_{j=1}^m a_j \norm{w - w^\star}^2 + \frac{\nu^2}{2}\sum_{j=1}^m a_j\norm{w^\star -  w_j^\star}^2  + F_j(w_j^\star)\\
    &= \frac{\nu^2}{2}\norm{w - w^\star}^2 + F(w^\star)
\end{align*}
From the bias-variance decompostion, we split the terms of $\norm{w - w_j^\star}^2$. We define $w^\star = \sum_{j=1}^m a_j w_j^\star$. From the above expression, we can see that $w^\star$ is the minimizer of $F(w)$ and $F(w^\star) = \frac{\nu^2}{2}\sum_{j=1}^m a_j \norm{w^\star - w_j^\star}^2$.

If $w_{j_i}^\star$ be the correct model for datapoint $i\in [n]$, then the closed form expression for the empirical risk minimizer of the square loss $\hat{w}$  is given by,
\begin{align*}
    \hat{w} = (\frac{1}{n}\sum_{i=1}^n x_i x_i^\intercal)^{-1}( \frac{1}{n}\sum_{i=1}^n x_i x_i^\intercal w_{j_i}^\star + x_i \xi_{i})
\end{align*}

We can now bound the out-of-sample prediction error of $\hat{w}$,
\begin{align*}
F(\hat{w}) - F(w^\star) &\leq 
    \frac{\nu^2}{2} \norm{\hat{w} - w^\star}^2 =  \frac{\nu^2}{2} \norm{(\frac{1}{n}\sum_{i=1}^n  x_{i}x_{i}^\intercal)^{-1}(\frac{1}{n}\sum_{i=1}^n (x_{i}x_{i}^\intercal (w_{j_i}^\star-w^\star) + x_{i} \xi_{i})}^2\\
    &\leq \underset{\text{Noise Term}}{\underbrace{\nu^2\norm{(\frac{1}{n}\sum_{i=1}^n x_{i}x_{i}^\intercal)^{-1}(\frac{1}{n}\sum_{i=1}^n  x_{i} \xi_{i})}^2}}  + \underset{\text{Bias term}}{\underbrace{\nu^2\norm{(\frac{1}{n}\sum_{i=1}^n x_{i}x_{i}^\intercal)^{-1}\frac{1}{n}\sum_{i=1}^n x_{i}x_{i}^\intercal (w_{j_i}^\star-w^\star)}^2}}
\end{align*}
We use $\norm{a+b}^2 \leq 2\norm{a}^2 + 2\norm{b}^2$ to split the error into bias and noise terms. The noise term does not depend on $w_{j_i}^\star$ and the bias term does not depend on the noise $\xi_i$.

We bound these two terms separately
\paragraph{Bound on the Noise Term}
From \cite[Theorem~10]{pmlr-v23-hsu12}, with probability $1-\delta/2$, and $n \geq 6 d(\log d + \log\frac{6}{\delta})$, we obtain
\begin{align*}
    \norm{\frac{1}{\nu^2 n}\sum_{i=1}^n x_{i}x_{i}^\intercal} \leq (1 - \rho')\quad    \text{where } \rho' =& \sqrt{\frac{4d(\log d + \log (6/\delta))}{n}} + \frac{2d(\log d + \log (6/\delta))}{3n}\\
\text{and } \nu^2\norm{(\frac{1}{n} \sum_{i=1}^n x_{i}x_{i}^\intercal)^{-1}(\frac{1}{n}\sum_{i=1}^n  x_{i} \xi_{i})}^2  &\leq \frac{1}{1-\rho'}\frac{\zeta^2(d + 2\sqrt{d \log(6/\delta)} + 2 \log(6/\delta))}{n}
\end{align*}

\paragraph{Upper Bound on the Bias Term}
By using Cauchy Schwarz, we separate the term $\norm{\frac{1}{\nu^2 n}\sum_{i=1}^n x_i x_i^\intercal}^2$.
\begin{align*}
    \nu^2\norm{(\frac{1}{n}\sum_{i=1}^n x_{i}x_{i}^\intercal)^{-1}\frac{1}{n}\sum_{i=1}^n x_{i}x_{i}^\intercal (w_{j_i}^\star-w^\star)}^2 \leq \norm{(\frac{1}{\nu^2 n}\sum_{i=1}^n  x_{i}x_{i}^\intercal)^{-1}}^2\cdot\norm{\frac{1}{n\nu}\sum_{i=1}^n x_{ij}x_{ij}^\intercal (w_{j_i}^\star-w^\star)}^2
\end{align*}
From ~\cite[Theorem~10]{pmlr-v23-hsu12}, we know that with probability atleast $1-\delta/2$, we have,
\begin{align}\label{eq:lr_bias_int}
    \norm{(\frac{1}{\nu^2 n}\sum_{i=1}^m \sum_{j=1}^n x_{ij}x_{ij}^\intercal)^{-1}}^2 \leq \frac{1}{(1-\rho')^2}
\end{align}

Therefore, to obtain a suitable bound on the bias term, we need to bound $\norm{Z'}^2$, where $Z' =\frac{1}{n\nu}\sum_{i=1}^n x_{i}x_{i}^\intercal (w_{j_i}^\star-w^\star) = \frac{1}{n\nu}\sum_{i=1}^n x_{i}x_{i}^\intercal \bar{w}_i$ where $\bar{w}_i = w_{j_i}^\star - w^\star$.

First, note that we need to bound the norm of a $0$-mean random vector. Let $W = \{w_1^\star,w_2^\star, \ldots, w_j^\star\}$.
\begin{align*}
    \E_{W,\cD_x}[\frac{1}{n\nu}\sum_{i=1}^n x_{i}x_{i}^\intercal (w_{j_i}^\star-w^\star)] =\E_{W}[\E_{\cD_x}[\frac{1}{n\nu}\sum_{i=1}^n x_{i}x_{i}^\intercal (w_{j_i}^\star-w^\star) | W]] =  \frac{\nu}{n}\sum_{i=1}^m\E_{W}[w_{j_i}^\star - w^\star] = 0
\end{align*}
To show some concentration, we need to write down the Moment Generating Function for this random variable. This procedure however, does not yield a term decreasing in the number of samples $n$. To obtain sharper decay, we need to carefully separate out the contributions due to the  two different probability distributions,
 namely that of $\bar{w}_i$ and $x_i$. 

Define $Z = n Z'$ and $\bar{w} = \frac{1}{n}\sum_{i=1}^n \bar{w}_i$.
Then, using $\norm{a+b}^2 \leq 2\norm{a}^2 + 2\norm{b}^2,\forall a,b\in \R^d$, we obtain,
\begin{align*}
     \norm{Z'}^2  =   \norm{\frac{Z}{n}}^2 \leq 2\norm{\frac{Z - n \nu \bar{w}}{n}}^2 + 2 \nu^2\norm{\bar{w}}^2
\end{align*}
Note that the second term does not depend on $x_i$ and only on the distribution of $\bar{w}_i$. Lemma~\ref{lem:bias_subexp_lr} shows that the first term depends on only the distribution of $x_i$ and does not depend on individual $\bar{w}_i$. Precisely, $Z = \sum_{i=1}^n x_{i}x_{i}^\intercal \bar{w}_i/\nu - n\nu \bar{w}$ is a sub-Exponential random variable.

\begin{lemma}\label{lem:bias_subexp_lr}
If $x_{i}\overset{iid}{\sim} \cN(0,\nu^2\I_d)$, then the random vector $Z = \frac{1}{n}\sum_{i=1}^n x_{i}x_{i}^\intercal \bar{w}_i/\nu$, we have   
\begin{equation}
    \E[\exp(\ip{u}{Z - n\nu \bar{w}})]\leq  \exp(2n\nu^2 \Delta_w^2\norm{u}^2),\quad 
\end{equation}
where $\norm{u} \leq \frac{1}{4\nu \Delta_w}$. Therefore, $Z - n\nu \bar{w}$ is a $(4n\nu^2 \Delta_w^2, 4\Delta_w^2)$ sub-Exponential random variable.
\end{lemma}

\begin{proof}
We first evaluate the moment generating function of $Z - n\nu \bar{w}$, conditioning on the choice of $W = \{w_i^\star\}_{i=1}^m$

\begin{align*}
    \E[\exp(\ip{u}{Z-n\nu \bar{w}})|W] &= \exp(-n\nu\ip{u}{\bar{w}})\prod_{i=1}^{n} \E[\exp(u^\intercal x_{i}x_{i}^\intercal \bar{w}_i/\nu)]\\
    &= \exp(-n\nu\ip{u}{\bar{w}})\prod_{i=1}^{n} \E[\exp(u^\intercal x_{i}x_{i}^\intercal \bar{w}_i/\nu)]
\end{align*}
Since $x_{i} \sim \cN(0,\nu^2\I_d)$,  conditioned on the set $W$, $x_{i}x_{i}^\intercal$ is a Wishart distribution with $1$ degree of freedom and mean $\nu^2\I_d$. Therefore, $u^\intercal x_{i} x_{i}^\intercal \bar{w}_i = \nu^2 \ip{u}{\bar{w}_i}Y_{i}$, where $Y_{i}$ follows a $\chi^2$ distribution with $1$ degree of freedom. Since $x_{i}$ are iid, therefore, $Y_{i}$ are also iid.

\begin{align*}
    \E[\exp(\ip{u}{Z - n\nu\bar{w}}) | W] &= \exp(-n\nu\ip{u}{\bar{w}})\prod_{i=1}^{n} \E[\exp(\nu\ip{u}{\bar{w}_i}Y_{i})|W]\\
    &=\exp(-n\nu\ip{u}{\bar{w}})\prod_{i=1}^{n} (1 - 2\nu\ip{u}{\bar{w}_i})^{-1/2},\quad \text{ where } \nu\ip{u}{\bar{w}_i} \leq \frac{1}{2} ,\,\forall i \in [n]
\end{align*}
We use the fact that for a $\chi^2_1$ random variable $Y$ $\E[\exp(tY)] = (1 - 2t)^{-1/2}$ for $t < 1/2$. To simplify the above inequality, we take log on both sides.
\begin{align*}
    \log(\E[\exp(\ip{u}{Z - n\nu\bar{w}})|W]) &= -n\nu\ip{u}{\bar{w}}-\frac{1}{2}\sum_{i=1}^{n}\log(1 - 2\nu\ip{u}{\bar{w}_i})\\
    &\leq -n\nu\ip{u}{\bar{w}} + \frac{1}{2}\sum_{i=1}^n (2\nu\ip{u}{\bar{w}_i} + 4\nu^2\ip{u}{\bar{w}_i}^2),\hspace{0.1cm}\text{ for }  \nu\max_{i\in [m]}\ip{u}{\bar{w}_i} \leq \frac{1}{4}\\
    &\leq  2\nu^2\sum_{i=1}^n\ip{u}{\bar{w}_i}^2 \leq 2n\nu^2 \Delta_w^2 \norm{u}^2
\end{align*}
We now simplify the bound on $\ip{u}{w_i^\star -w^\star}$ to obtain the sub-Exponential condition. We use $\log( 1- x) \geq x - x^2 , x \leq 0.68$. Further, we cancel the terms of $\bar{w}$ from both sides and bound $\norm{\bar{w}_i}^2$ by $\Delta_w^2$. Note that to obtain $\E[\exp(\ip{u}{Z}]$ by removing the log, as upper bound is does not depend on $W$.

Note that a sufficient condition on $u$ is obtained by using Cauchy-Schwarz,
\begin{align*}
    \nu\ip{u}{\bar{w}_i} &\leq \nu \max_{i\in [n]}\norm{\bar{w}_i} \norm{u}\leq \frac{1}{4}\\
    \norm{u}&\leq \frac{1}{4\nu \max_{i\in [n]}\norm{\bar{w}_i}} \leq \frac{1}{4\nu \Delta_w} 
\end{align*}
\end{proof}

To bound the term of $\norm{Z - n\nu \bar{w}}$, we need a bound on the norm of a $0$-mean $(v^2,\alpha)$-sub-Exponential random variable.
Lemma~\ref{lem:subexp_norm_conc} provides us this bound. Its proof follows the bound on norm of sub-Gaussian random variables provided in ~\cite{norm-subgaussian}.

Therefore, from Lemma~\ref{lem:subexp_norm_conc}, with probability $1-\delta/8$, we have , 
\begin{align*}
    \norm{\frac{Z - n\nu \bar{w}}{n}}^2 \leq 32\frac{\nu^2\Delta_w^2 d}{n}\log\left(\frac{32}{\delta}\right) 
\end{align*}

To handle the contribution of the term $\bar{w}_i$, note that it is the mean of $n$ random variables each of which has a bounded norm.
\begin{lemma}\label{lem:w_conc}
If $\norm{\bar{w}_i}\leq \Delta_w$, then with probability $1-\delta/8$, we have,
\begin{align*}
    \norm{\bar{w}}^2 \leq \frac{d\Delta_w^2}{2n} \log\left(\frac{8}{\delta}\right)
\end{align*}
where $\bar{w} = \frac{1}{n}\sum_{i=1}^n \bar{w}_i = \frac{1}{n} \sum_{i=1}^n w_{j_i}^\star - w^\star$.    
\end{lemma}
\begin{proof}
    
 As $\bar{w}_i$ is a $0$-mean random variable whose norm is bounded by $\Delta_w$, $\bar{w}_i$ is a $(\Delta_w^2/4)$-sub-Gaussian random vector. Therefore, as $\bar{w} = \frac{1}{n}\sum_{i=1}^n \bar{w}_i$, $\bar{w}$ is a $\frac{\Delta_w^2}{4n}$ sub-Gaussian random variable.
 
 From ~\cite{norm-subgaussian}, we obtain bounds on the tail of norm of a sub-Gaussian distribution. By setting the tail probability as $1/8$ we finish the proof. 
\end{proof}

\paragraph{Completing the proof}
Combining the bounds for $\bar{w}$ and $Z - n\nu\bar{w}$, we obtain with probability $1 - \delta/2$
\begin{align*}
    \norm{\frac{Z}{n}}^2 \leq 2\norm{\frac{Z - n \nu \bar{w}}{n}}^2 + 2 \nu^2\norm{\bar{w}}^2 \leq \mathcal{O}\left(\frac{d\nu^2 \Delta_w^2}{n}\log\left(\frac{1}{\delta}\right)\right)
\end{align*}

\subsection{Proof of Remark~\ref{rem:kl}}
We can split $KL$ divergence between two base distributions $\cD_j$ and $\cD_{j'}$ as the sum of $KL$ divergences of the covariate distribution $\Pr[x]$, the label distribution $\Pr[y|x]$. For the mixture of hyperplanes, two base distributions differ in only $\Pr[y|x]$, therefore the $KL$ divergences due to $\Pr[x]$ should be $0$. If the noise distribution is $\cN(0,\zeta^2)$, then $\Pr[y|x]$ is also a gaussian with mean $\ip{w_j^\star}{x}$ for the $j^{th}$ base distribution. As $KL(\cN(\mu_1,\zeta^2)|\cN(\mu_2,\zeta^2)) = \frac{\norm{\mu_1 - \mu_2}^2}{2\nu^2}$. In our case, for base distributions $j$ and $j'$,
\begin{align*}
    KL(\cD_i,\cD_j) &= \frac{\ip{x}{w_{j}^\star - w_{j'}^\star}^2}{2\zeta^2}\\
    \E_x[KL(\cD_i,\cD_j)]
    &= \frac{\nu^2\norm{w_i^\star - w_j^\star}^2}{2\zeta^2} \leq \frac{\nu^2\Delta_w^2}{2\zeta^2}
\end{align*}

\subsection{Additional Technical Lemmas}

\begin{lemma}[Projections on a convex set]\label{lem:cvx_proj}
If $\cW$ is a convex and bounded set, then $\forall w_1 \notin \cW$ and $w_2\in \cW$,
\begin{equation}
    \ip{w_1 -proj_{\cW}(w_1)}{w_2 - proj_{\cW}(w_1)}\leq 0
\end{equation}
\end{lemma}
Assume by contradiction that $\exists w_1 \notin \cW$ and  $w_2 \in \cW$ such that $\ip{w_1 -proj_{\cW}(w_1)}{w_2 - proj_{\cW}(w_1)} > 0$. Let $\hat{w_1} = proj_{\cW}(w_1)$. Then, we obtain,
\begin{align*}
    &\ip{w_1 - \hat{w_1}}{w_2 - \hat{w_1}} >0 \\
    &\norm{w_1 - \hat{w_1}}^2 + \norm{w_2 - \hat{w_1}}^2 > \norm{w_1 - w_2}^2
\end{align*}
We use $2\ip{a}{b} = \norm{a}^2 + \norm{b}^2 - \norm{a-b}^2,\, \forall a,b \in \R^d$. This implies that the angle between $w_1 - \hat{w_1}$ and $w_2 - \hat{w_1}$ is acute. By the definition of projection, $\norm{w_1 - \hat{w_1}}^2 \leq \norm{w_2 - w_1}^2$, therefore by property of triangles, angle between $w_1 - w_2$ and $\hat{w_1} - w_2$ is also acute. Thus, the triangle formed by $w_1,w_2$ and $\hat{w_1}$ is acute. Therefore, $\exists w' \in \cW$ on the line joining $w_2$ and $\hat{w_1}$ which such that $w_1 - w'$ is perpendicular to $w_2  - \hat{w_1}$. Since the set is convex, so $w'\in \cW$. Thus, we obtain,
\begin{align*}
    &\ip{w_1 - w'}{\hat{w_1} - w'} = 0 \\
    &\norm{w_1 - w'}^2 + \norm{\hat{w_1} - w'}^2 = \norm{\hat{w_1} - w_1}^2\\
    &\norm{\hat{w_1} - w'}^2 < \norm{\hat{w_1} - w_1}^2
\end{align*}
We again use $2\ip{a}{b} = \norm{a}^2 + \norm{b}^2 - \norm{a-b}^2,\, \forall a,b \in \R^d$. For the final step, since angle between $w_1 -\hat{w_1}$ and $w_2 - \hat{w_1}$ is acute, $\norm{\hat{w_1} - w'}^2 >0$. The last inequality is a contradiction to the projection property as $\hat{w_1} = \min_{w\in\cW}\norm{w_1 - w}^2$. Therefore, by contradiction $\ip{w_1 - \hat{w_1}}{w_2 - \hat{w_1}}\leq 0$.

\begin{lemma}[Sub-Exponential Norm Concentration]\label{lem:subexp_norm_conc}
If $Z\in \R^d$ is a $0$-mean $(v^2,\alpha)$-sub-Exponential random variable, then
\begin{equation}
    \Pr[\norm{Z} \geq t] \leq \exp(-\frac{1}{4}\min\{\frac{t^2}{2 v^2 d},\frac{t}{\alpha d}\})
\end{equation}
\end{lemma}
\begin{proof}
First, let $u(Z) = Z/\norm{Z}$. Now, let $\mathcal{Z}$ be a $\frac{1}{2}$-cover of the unit sphere in $d$ dimensions. Then, we can always find a $u_{k(Z)} \in \mathcal{Z}$ such that $\norm{u(Z) - u_{k(Z)}}\leq \frac{1}{2}$.

Then, we have,
\begin{align*}
\ip{u(Z)}{Z}  = \norm{Z} = \ip{u_{k(Z)}}{Z} + \ip{u(Z) - u_{k(Z)}}{Z}\leq\ip{u_{k(Z)}}{Z}  + \norm{Z}/2
\end{align*}
Therefore, $\norm{Z}\leq 2\ip{u_{k(Z)}}{Z}$. Since the size of $\cZ$ is $4^d$ as it is a $\frac{1}{2}$-cover. Now, by using the union bound over all elements in $\cZ$, we obtain,
\begin{align*}
    \Pr[\norm{Z} \geq t] \leq \Pr[\ip{u_{k(Z)}}{Z} \geq t/2] \leq \Pr[\exists u \in \cZ, \ip{u}{Z} \geq t/2]\leq 4^d \exp(-\frac{1}{4}\min\{\frac{t^2}{2 v^2},\frac{t}{\alpha}\}
\end{align*}
If we set $s = \frac{1}{4}\min\{\frac{t^2}{2v^2},\frac{t}{\alpha}\} - \log(4)$, we have,
\begin{align*}
    \Pr[\norm{Z} \geq t] \leq 4^d \exp(-\frac{1}{4}\min\{\frac{t^2}{2 v^2},\frac{t}{\alpha}\} \leq \exp(-s)
\end{align*}
\end{proof}

\end{document}